\documentclass[11pt]{article}
\usepackage[utf8]{inputenc}
\usepackage[T1]{fontenc}    
\usepackage[colorlinks,citecolor=blue]{hyperref}       
\usepackage{url}            
\usepackage{bm}

\usepackage{graphicx} 
\graphicspath{ {./imgs/} }
\usepackage{caption}
\usepackage{subcaption}
\usepackage{float}

\usepackage{amsmath,eqnarray,easybmat,amssymb,amsthm,amsfonts,nicefrac} 

\usepackage{algorithmic,algorithm} 

\usepackage{color}

\usepackage[round]{natbib}
\usepackage{multirow}
\usepackage{dsfont}

\usepackage{style}

\usepackage{tikz}

\usepackage[margin=1in]{geometry}

\def\Dd{\mathcal{D}}

\renewcommand{\x}{\z}

\usepackage{rotating}

\usepackage{pdflscape}

\newcommand{\Sfrak}{\mathfrak{S}}

\title{Private Stochastic Convex Optimization:\\ Efficient Algorithms for Non-smooth Objectives}
\author{
  Raman Arora\\
  \texttt{arora@cs.jhu.edu}
  \and
  Teodor V. Marinov\\
  \texttt{tmarino2@jhu.edu}
  \and
  Enayat Ullah\\
  \texttt{enayat@jhu.edu}\\
  \and
  \centering Johns Hopkins University
}
\date{}

\begin{document}

\maketitle

\begin{abstract}
In this paper, we revisit the problem of private stochastic convex optimization. We propose an algorithm based on noisy mirror descent, which achieves optimal rates both in terms of statistical complexity and number of queries to a first-order stochastic oracle in the regime when the privacy parameter is inversely proportional to the number of samples.
\end{abstract}

\section{Introduction}
\label{sec:intro}
Modern machine learning systems often leverage data that are generated ubiquitously and seamlessly through devices such as smartphones, cameras, microphones, or user's weblogs, transaction logs, social media, etc. Much of this data is private, and releasing models trained on such data without serious privacy considerations can reveal sensitive information~\citep{NS08, Sweeney05}.
Consequently, much emphasis has been placed in recent years on machine learning under the constraints of a robust privacy guarantee. One such notion that has emerged as a de facto standard is that of differential privacy. 

Informally, differential privacy provides a quantitative assessment of how different are the outputs of a randomized algorithm when fed two very similar inputs. If small changes in the input do not manifest as drastically different outputs, then it is hard to discern much information about the inputs solely based on the outputs of the algorithm. In the context of machine learning, this implies that if the learning algorithm is not overly sensitive to any single datum in the training set, then releasing the trained model should preserve the privacy of the training data. This requirement, apriori, seems compatible with the goal of learning, which is to find a model that generalizes well on the population and does not overfit to the given training sample. It seems reasonable then to argue that privacy is not necessarily at odds with generalization, especially 
when large training sets are available.

We take the following stochastic optimization view of machine learning, 
where the goal is to find a predictor that minimizes the expected loss (aka risk)
\begin{equation}
\label{eq:stoch_opt_problem}
\min_{\w \in \cW} F(\w)=\mathbb{E}_{\x \sim \Dd}[f(\w,\x)], 
\end{equation}
based on i.i.d. samples from the source distribution $\Dd$, and full knowledge of the instantaneous objective function $f(\cdot,\cdot)$ and the hypothesis class $\cW$. 
We are particularly interested in convex learning problems where the hypothesis class is a convex set and the loss function $f(\cdot, \x)$ is a convex function in the first argument for all $\x \in \cZ$. 
We seek a learning algorithm that uses the smallest possible number of samples and the least runtime and returns $\tilde{\w}$ such that $F(\tilde{\w})\leq \inf_{\w \in \cW}F(\w) + \alpha$, for a user specified $\alpha > 0$, while guaranteeing differential privacy (see Section~\ref{sec:prelimDP} for a formal definition). 

A natural approach to solving Problem~\ref{eq:stoch_opt_problem} is sample average approximation (SAA), or empirical risk minimization (ERM), where we instead minimize an empirical approximation of the objective based on the i.i.d. sample. Empirical risk minimization for convex learning problems has been studied in the context of differential privacy by several researchers including~\cite{bassily2019private} who give statistically efficient algorithms with oracle complexity matching that of optimal non-private ERM. 

An alternative approach to solving Problem~\ref{eq:stoch_opt_problem} is stochastic approximation (SA), wherein rather than form an approximation of the objective, the goal is to directly minimize the true risk. The learning algorithm is an iterative algorithm that processes a single sample from the population in each iteration to perform an update. Stochastic gradient descent (SGD), for instance, is a classic SA algorithm. Recent work of \cite{feldman2019private} gives optimal rates for convex learning problems (Problem~\ref{eq:stoch_opt_problem}) using stochastic approximation for smooth loss functions; however, they leave open the question of optimal rates for non-smooth convex learning problems which include a large class of learning problems, including, for example, support vector machines. In this work, we focus on non-smooth convex learning problems and propose a simple algorithm which achieves nearly optimal rates in the special case where the privacy guarantee scales inversely with the number of observed samples. There are alternative approaches\footnote{We became aware of it in a private communication with Raef Basilly.}; for further discussion we refer the reader to Section~\ref{sec:other_appr}. 

\section{Notation and Preliminaries}
We consider the general learning setup of \cite{vapnik2013nature}. Let $\cZ$ be the sample space and let $\cD$ be an unknown distribution over $\cZ$. We are given a sample $\x_1,\x_2,\ldots, \x_n$ drawn identically and independently (i.i.d) from $\cD$ --  the samples $\x_i$ may correspond to (feature, label) tuples as in supervised learning, or to features in  unsupervised learning. We assume that loss  $f:\cW \times \cZ \rightarrow \R$ is a convex function in its first argument $\w$ and the hypothesis set $\cW$ is a convex set. Given n samples, the goal is to minimize the population risk (Problem \ref{eq:stoch_opt_problem}).

We assume that the hypothesis class $\mathcal{W}$ is a convex, closed and bounded set in $\R^d$ equipped with norm $\norm{\cdot}$, such that $\|\w\| \leq D$ for all $\w \in \cW$. 
We recall that the dual space of $(\R^d, \norm{\cdot})$ is the set of all linear functionals over it; the dual norm, denoted by $\norm{\cdot}_*$, is defined as $\norm{h}_* = \min_{\norm{w}\leq 1} h(\w)$, where $h$ is a element of the dual space $(\R^d,\norm{\cdot}_*)$.
In general we will use $\|\cdot\|$ to denote the $\ell_2$ norm when there is no risk of confusion. 

We do not assume that $f(\w,\x)$ is necessarily differentiable and will denote an arbitrary sub-gradient as $\nabla f(\w,\x)$.
We assume that $f(\cdot,\x)$ is $L$-lipschitz with respect to the dual norm $\norm{\cdot}_*$, i.e.,  $\abs{f(\w_1,\x)-f(\w_2,\x)}\leq L\norm{\w_1-\w_2}$  for all $\x$. For convex $f$, this implies that sub-gradients with respect to $\w$, are bounded as  $\norm{\nabla f(\w,\x)}_* \leq L, \ \forall \ \w \in \mathcal{W},\x \in \mathcal{Z}$. A popular instance of the above is the $\ell_p/\ell_q$ setup, which considers $\ell_p$ norm in primal space and the corresponding $\ell_q$ norm in the dual space such that $\frac{1}{p}+\frac{1}{q} =1$.
A function $g(\cdot)$ is $\beta$-strongly smooth (or just $\beta$-smooth) if $\norm{\nabla g(\w_1)-\nabla g(\w_2)}_* \leq \beta \norm{\w_1-\w_2}, \forall \w_1,\w_2 \in \cW$. For convex functions, this is equivalent to the condition $g(\w_2) \leq g(\w_1) + \langle \nabla g(\w_1),\w_2- \w_1 \rangle + \frac{\beta}{2}\|\w_1-\w_2\|^2, \forall \w_1,\w_2\in\mathcal{W}$.
A function $g$ is $\lambda$-strongly convex if $g(\w_2) \geq g(\w_1) + \langle \nabla g(\w_1), \w_2 - \w_1\rangle + \frac{\lambda}{2}\norm{\w_1 - \w_2}^2, \forall \w_1, \w_2 \in \cW$. Finally, we use $\tilde O(\cdot)$ to suppress poly-logarithmic factors in the complexity.

\subsection{Mirror descent and potential functions}
We recall some basics of convex duality, which will help us set up the mirror descent algorithm and analysis.
For a convex function $\Phi:\mathbb{R}^d \rightarrow \mathbb{R}$, we define its conjugate $\Phi^*:\mathbb{R}^d \rightarrow \mathbb{R} \cup \{\infty\}$ as $\Phi^*(Y) = \sup_{X} \langle Y,X \rangle - \Phi(X)$. Moreover, $D_{\Phi}(\x||\y)$ denotes Bregman divergence induced by $\Phi$, defined as
\begin{align*}
    D_{\Phi}(\x||\y) = \Phi(\x) - \Phi(\y) - \langle \nabla \Phi(\y), \x - \y \rangle.
\end{align*}
For a convex set $\cW$, we denote by $\I_{\cW}$ the indicator function of the set $\cW$, $\I_{\cW}(\x) = 0 \ \text{ if }  \ \x \in \cW$, and $\infty$ otherwise. 
The following result from~\cite{kakade2009duality} establishes a natural relation between strong convexity and strong smoothness of a potential function and its conjugate, respectively.
\begin{theorem}[Theorem 6 from~\cite{kakade2009duality}]
\label{thm:kakade_sconv_ssmooth}
Assume $\Phi$ is a closed convex function. Then $\Phi$ is $\alpha$-strongly convex with respect to a norm $\|\cdot\|$ iff $\Phi^*$ is $\frac{1}{\alpha}$-strongly smooth with respect to the dual norm $\|\cdot\|_*$.
\end{theorem}

\subsection{Differential privacy} 
\label{sec:prelimDP}
We seek to design algorithms for solving the stochastic convex optimization problem (Problem \ref{eq:stoch_opt_problem}) with the additional constraint that the algorithm's output is differentially private. 

\begin{definition}[$(\epsilon,\delta)$-differential privacy~\citep{dwork2006calibrating}] An algorithm $\mathcal{A}$ satisfies $(\epsilon,\delta)$-differential privacy if given two datasets $S$ and $S'$, differing in only one data point, it satisfies that for any measurable event $R \in \text{Range}(\cA)$
\begin{align*}
    \mathbb{P}(\mathcal{A}(S)\in R) \leq e^\epsilon \mathbb{P}(\mathcal{A}(S')\in R) + \delta.
\end{align*}
\end{definition}

\section{Related Work}
In convex learning and optimization, the following four classes of functions are widely studied: (a) $L$-lispchitz convex functions, (b)  $\beta$-smooth and convex functions, (c)  $\lambda$-strongly convex functions, and (d)  $\beta$-smooth and $\lambda$-strongly convex functions. 
In the computational framework of first order stochastic oracle, algorithms with optimal oracle complexity for all these classes of functions have long been known~\citep{agarwal2009information}. However, the landscape of known results changes with the additional constraint of privacy. 
We can trace two approaches to solving the private version of Problem~\ref{eq:stoch_opt_problem}. The first is private ERM \citep{CMS11,bassily2014private,Feldman2018PrivacyAB,bassily2019private} and the second is private stochastic approximation \citep{feldman2019private}.
\cite{CMS11} begin the study of private ERM by constructing algorithms based on output perturbation and objective perturbation. Under the assumption that the stochastic gradients are $\beta$-Lipschitz continuous, the output perturbation bounds achieve excess population risk of $O(LD\max(1/\sqrt{n}, d/(n\epsilon),d\sqrt{\beta}/(n^{2/3}\epsilon)))$, where $L$ is the Lipschitz constant of the loss function and $D$ measures the diameter of the hypothesis class in the respective norm. The objective perturbation bounds have a similar form.
\cite{bassily2014private} showed tight upper and lower bounds for the excess empirical risk. They also showed bounds for the excess population risk when the loss function does not have Lipschitz continuous gradients, achieving a rate of $O(d^{1/4}/(\sqrt{n}\epsilon))$. Their techniques appeal to uniform convergence i.e. bounding $\sup_{\w \in \cW} F(\w) - \hat F(\w)$, and convert the guarantees on excess empirical risk to get a bound on the excess population risk.
These guarantees, however, were sub-optimal. \cite{bassily2019private} improved these to get optimal bounds on excess population risk, leveraging connections between algorithmic stability and generalization. The algorithms given by \cite{bassily2019private} are sample efficient but their runtimes are superlinear (in the number of samples), whereas the non-private counterparts run in linear time. In a follow-up work, \cite{feldman2019private} improved the runtime of some of these algorithms without sacrificing statistical efficiency; however, the authors require that the stochastic gradients are Lipschitz continuous. 
Essentially, the statistical complexity of private stochastic convex optimization has been resolved, but some questions about computational efficiency still remain open.  We begin with a discussion of different settings for the population loss in subsequent paragraphs, describe what is already known, and what are the avenues for improvement.

\paragraph{Non-smooth Lipschitz Convex.}
For the class of $L$-Lipschitz convex functions, \cite{bassily2014private} improved upon \cite{CMS11} and gave optimal bounds on excess empirical risk of $O\left(\frac{d}{\epsilon n}\right)$. They then appeal to uniform convergence to convert the guarantees on excess empirical risk to get an excess population risk of $O\left(\max\left(\frac{d^{1/4}}{\sqrt{n}},\frac{\sqrt{d}}{\epsilon n}\right)\right)$. This is sub-optimal and was very recently improved by \cite{bassily2019private} using connections between algorithmic stability and generalization to get $O\left(\max\left(\frac{1}{\sqrt{n}},\frac{\sqrt{d}}{\epsilon n}\right)\right)$. This is optimal since $\frac{1}{\sqrt{n}}$ is the optimal excess risk without privacy constraints, and $\frac{\sqrt{d}}{\epsilon n}$ is the optimal excess risk when the data distribution is the empirical distribution. This resolves the statistical complexity of private convex learning and shows that in various regimes of interest, e.g., when $d=\Theta(n)$ and $\epsilon = \Theta(1)$, the constraint of privacy has no effect on utility.
However, the algorithm  proposed by \cite{bassily2019private} is based on Moreau smoothing and proximal methods, and requires $O(n^5)$ stochastic gradient computations. This rate vastly exceeds the gradient computations needed for non-private stochastic convex optimization which are of the order $O(n)$. The computational complexity was improved by \cite{feldman2019private} to $O(n^2)$ by using a regularized ERM algorithm that runs in phases and after each phase, noise is added to the solution (output perturbation) and used as regularization for the next phase. 
 
\paragraph{Smooth Lipschitz Convex.} For $\beta$-smooth convex $L$-Lispschitz functions, \cite{bassily2019private} give an algorithm with optimal bounds on  excess risk of $O\left(LD\max\left\{\frac{1}{\sqrt{n}},\frac{\sqrt{d}}{\epsilon n}\right\}\right)$ with $O\left(\min\left\{n^{3/2},\frac{n^{5/2}}{d}\right\}\right)$ queries to the stochastic gradient oracle. This, again, was improved in a later work of \cite{feldman2019private} to $O(n)$ stochastic gradient queries. Note that even for non-private stochastic optimization $O(n)$ stochastic gradient computations are necessary, so \cite{feldman2019private} achieve optimal statistical and oracle complexity for the smooth Lipschitz convex functions. The algorithm they present is an instance of a single-pass noisy SGD, and the guarantees hold for the last iterate. 

 \paragraph{(Smooth and Non-smooth) Lipschitz Strongly Convex.} For $L$-Lispchitz $\lambda$-strongly convex functions, \cite{bassily2014private} gave an algorithm with optimal excess empirical risk which is in $\tilde O\left(\frac{d}{\lambda n^2 \epsilon^2}\right)$. However, as in the non-strongly convex case, the corresponding excess population risk is $\tilde O\left(\frac{\sqrt{d}}{\lambda \epsilon n}\right)$, which is sub-optimal. For this case, \cite{feldman2019private} proposed an algorithm which achieves optimal rates in $O(n)$ time for smooth functions but $O(n^2)$ for non-smooth functions.

\section{Algorithm and Utility Analysis}
In this section, we describe the proposed algorithm and provide convergence guarantees. 

Recall that we are given a set of $n$ samples~$(\x_1,\ldots,\x_n)$ drawn i.i.d. from $\cD$. We consider an iterative algorithm, wherein, at time $t$,  we sample index $Y_t$ uniformly at random from the set of indices, $[n]$. Note that $Y_t$ is a random variable; we denote the realization of $Y_t$ as $y_t$. Through the run of the algorithm, we maintain a set, $F_t$, of indices observed until time $t$, i.e., $F_t = \{ y_{\tau} \mid \tau < t \}.$

If $y_t \notin F_t$, i.e., $\x_{y_t}$ is a \emph{fresh} sample that has not been seen and processed prior to the $t^\textrm{th}$ iteration, then we proceed with a projected gradient descent step using the noisy gradient $\nabla f(\w_t,\x_{y_t}) + \xi_t$. If $y_t \in F_t$, then the algorithm simply perturbs the current iterate, $\w_t$, and projects on to the set $\cW$. We remark that the \emph{noise-only} step is important for the privacy analysis, as it allows for privacy amplification by sub-sampling.

The algorithm terminates when half of the points in the training set have been processed at least once, i.e., the size of $F_t$ is greater than or equal to $n/2$. We denote this stopping time by $\tau$. We refer the reader to Algorithm~\ref{alg:private_SGD} for the pseudo-code.

\begin{algorithm}[h]
\caption{\texttt{Private SGD}($\w_1,\{\eta_t\}_t,\{\xi_t\}_t,\{\x_t\}_t$)}
\label{alg:private_SGD}
\begin{algorithmic}[1]
\REQUIRE{Step size schedule $\{\eta_t\}_t$, noise sequence $\{\xi_t\}_t$, stream of data points $\{\x_t\}_t$, initial iterate $\w_1$}
\STATE $F_1 = \emptyset$,$\tau=1$
\WHILE{$|F_\tau| \leq n/2$}
\STATE Sample index $y_\tau$ uniformly from $[n]$
\IF{$\{y_\tau\} \bigcap F_\tau \neq \emptyset$}
\STATE $\w_{\tau+1} =\mathcal{P}(\w_\tau - \eta_\tau\xi_\tau)$
\STATE $F_{\tau+1} = F_\tau$
\ELSE
\STATE $\w_{\tau+1} = \mathcal{P}(\w_\tau - \eta_\tau (\nabla f(\w_\tau,\x_{y_\tau}) +\xi_\tau))$
\STATE $F_{\tau+1} = F_\tau \bigcup \{y_\tau\}$
\ENDIF
\STATE $\tau=\tau+1$
\ENDWHILE
\ENSURE{Average iterate $\hat\w_\tau = \frac{1}{|F_\tau|}\sum_{j \in F_\tau}\w_j$}
\end{algorithmic}
\end{algorithm}

Next, we establish the utility guarantee for Algorithm~\ref{alg:private_SGD}. We first show that $\tau$ is finite almost surely and bounded by $O(n)$ with probability $1- O(\exp{-n})$. The proof follows a standard bins-and-balls argument. 
\begin{theorem}
\label{thm:bins_and_balls}
For any $n\geq 16$ with probability $1-2\exp{-n/16}$ it holds that $\tau \leq 2n$, which implies $\mathbb{E}[\tau] \leq O(n)$.
\end{theorem}

We proceed with bounding the regret of Algorithm~\ref{alg:private_SGD}. Given a sequence of convex functions $\tilde f_t : \mathcal{W} \rightarrow \mathbb{R}$, where $f_t(\cdot) = f(\cdot, \z)$, we bound the regret 
\begin{align*}
    \bar R(\tau,u) = \sum_{t=1}^\tau \mathbb{E}[\tilde f_t(\w_t)] - \sum_{t=1}^\tau \mathbb{E}[\tilde f_t(\u)],
\end{align*}
where the expectation is with respect to any  randomness in the algorithm and randomness of $\tilde f_t$'s. We assume full access to the gradient of $\tilde f_t$ and that the diameter of $\mathcal{W}$ is $D$, i.e., $\forall \w,\v \in \mathcal{W}, \|\w - \v\| \leq D$. 

Our setup fits the popular framework of Online Stochastic Mirror Descent (OSMD) algorithm, wherein, given a strictly convex potential $\Phi: \mathbb{R}^d \rightarrow \mathbb{R}$, the updates are given as 
\begin{equation}
\label{eq:OSMD}
\begin{aligned}
    \tilde \w_{t+1} &= \nabla \Phi^*(\nabla \Phi(\w_t) - \eta\nabla \tilde f_t(\w_t))\\
    \w_{t+1} &= \argmin_{\w \in \cW} D_{\Phi}(\w || \tilde \w_{t+1}).
\end{aligned}
\end{equation}
We utilize the following result. 

\begin{theorem}[Theorem 5.5~\citep{bubeck2012regret}]
\label{thm:bubeck}
Let $f_1,\ldots,f_\tau$ be a sequence of functions from  $\mathbb{R}^d$ to $\mathbb{R}$. Suppose $\nabla \tilde f_t$ is an unbiased estimator of $\nabla f_t$, i.e., $\mathbb{E}[\nabla \tilde f_t] = \nabla f_t$. If one initializes OSMD~\ref{eq:OSMD} with $\w_1 = \argmin_{\w \in \cW} \Phi(\w)$, then the expected pseudo-regret $\bar R(\tau)$ is bounded as
\begin{align*}
    \bar R(\tau,\u) \leq \frac{D_{\Phi}(\u||\w_1)}{\eta} + \frac{1}{\eta} \sum_{t=1}^\tau \mathbb{E}[D_{\Phi^*}(\nabla f_t(\w_t) - \eta\nabla \tilde f_t(\w_t)||\nabla f_t(\w_t))].
\end{align*}
\end{theorem}

For any fixed $\tau < \infty$, we can view Algorithm~\ref{alg:private_SGD} as an instance of OSMD, with $\Phi \equiv \frac{1}{2}\|\cdot\|_2^2$, and $\tilde f_t(\cdot)$ defined as
\begin{align*}
    \tilde f_t(\cdot) = 
    \begin{cases}
        f(\cdot,\x_{y_t}) + \langle \xi_t,\cdot\rangle & y_t \not\in F_t\\
        \langle \xi_t,\cdot\rangle & \text{otherwise,}
    \end{cases}
\end{align*}
and $f_t(\w) = \mathbb{E}_{\xi_t}[\tilde f_t(\w)]$. Theorem~\ref{thm:kakade_sconv_ssmooth} implies that for any $\x,\y \in \mathbb{R}^d$ $D_{\Phi^*}(\x||\y) \leq 1/\alpha\|\x - \y\|_*^2$, where $\alpha$ is the strong-convexity parameter of the potential (in this case $\alpha=1$) -- using this result with Theorem~\ref{thm:bubeck}, and taking expectation with respect to the randomness in $\tau$ and $\xi_t, t\in [\tau]$, we have that for any $\u \in \cW$
\begin{equation}
\label{eq:reg_bound_eta}
\begin{aligned}
    \mathbb{E}\left[\sum_{t=1}^\tau f_t(\w_t)\right] - \mathbb{E}\left[\sum_{t=1}^\tau f_t(u)\right] &\leq \frac{\|\u - \w_1\|_2^2}{2\eta} + \frac{1}{2\eta}\left[\sum_{t=1}^\tau \eta^2\|\nabla \tilde f_t(\w_t)\|_2^2\right]\\
    &\leq \frac{\|\u - \w_1\|_2^2}{2\eta} + \eta \mathbb{E}\left[\sum_{t=1}^\tau \| \nabla f_t(\w_t)\|_2^2 + \|\xi_t\|_2^2\right].
\end{aligned}
\end{equation}
Since $\xi \sim \mathcal{N}(0, \sigma\I)$, we have the following corollary.

\begin{corollary}
\label{cor:reg_bound}
Suppose the diameter of $\cW$ is bounded by $D$ and that $f(\cdot,\x_t)$ is an $L$-Lipschitz function for all $t \in [n]$. Running Algorithm~\ref{alg:private_SGD} for $\tau$ iterations with noise sequence $\xi_t\sim \mathcal{N}(0, \sigma\I)$ and step size $\eta_t = \frac{D}{\sqrt{n}(L + \sigma\sqrt{d})}$ guarantees that
\begin{align*}
    \mathbb{E}\left[\sum_{t=1}^\tau f_t(\w_t)\right] - \mathbb{E}\left[\sum_{t=1}^\tau f_t(u)\right] \leq 2D(L + \sigma\sqrt{d})\sqrt{n}.
\end{align*}
\end{corollary}
\begin{proof}
Combine Equation~\ref{eq:reg_bound_eta}, Wald's lemma and the assumptions of the theorem to get 
\begin{align*}
    \mathbb{E}\left[\sum_{t=1}^\tau f_t(\w_t)\right] - \mathbb{E}\left[\sum_{t=1}^\tau f_t(u)\right] \leq \frac{D^2}{2\eta} + \eta(L^2+\sigma^2d)\mathbb{E}[\tau].
\end{align*}
Using the bound on $\mathbb{E}[\tau] \leq O(n)$ from Theorem~\ref{thm:bins_and_balls} together with the step-size choice in the corollary claim finishes the proof.
\end{proof}

The result now follows using the Corollary above and the following lemma.
\begin{lemma}
\label{lem:stopping_time_conv}
For the iterates $\{\w_t\}_{t=1}^\tau$ of Algorithm~\ref{alg:private_SGD} it holds that $\mathbb{E}[\sum_{t \in F_\tau} f(\w_t,\x_{Y_t})] \geq n\mathbb{E}[F(\hat \w_\tau)]$, where $\hat \w_\tau$ is the output of the algorithm.
\end{lemma}
\begin{proof}
Let $S_t = |F_t|$ be the random variable measuring the size of $F_t$. It then holds that $\tau$ is a stopping time for the process $\{S_t\}_{t=1}^\infty$, i.e., $\tau = \min \{t : S_t = n\}$. Further, denote by $\mathbf{y} = (y_1,\ldots,y_T)$ the vector of indices consisting of $y_1,\ldots,y_T$ for some $T$. Let us focus on the term $\mathbb{E}[\sum_{t \in F_\tau} f(\w_t,x_{Y_t})]$. Let $\mu = \mathcal{N}(\pmb{0},\sigma^2\mathbf{I}_d)^{T}\times\text{Unif}[n]^{T}$ be the measure capturing all the randomness after $T$ iterations except for  the randomness with respect to the data distribution $\mathcal{D}$. Furthermore, let $\chi(\cdot)$ denotes the indicator of the input event. 
We now show that $\mathbb{E}[\sum_{t \in F_\tau} f(\w_t,\x_{Y_t})] = \mathbb{E}\left[\sum_{t\in F_\tau} F(\w_t)\right]$, which essentially follows using the tower rule of expectation.
We have the following
\begin{align*}
    \mathbb{E}\left[\sum_{t \in F_\tau} f(\w_t,\x_{Y_t})\right] &= \mathbb{E}\left[\mathbb{E}\left[\sum_{t\in F_\tau} f(\w_t,\x_{\Y_t})|\tau,F_\tau\right]\right]\\
    &= \mathbb{E}\left[\sum_{\tau = T, F_T = \mathbf{y}} \chi(\tau = T, F_T = \mathbf{y}) \mathbb{E}\left[\frac{\chi(\tau = T, F_T = \mathbf{y})\sum_{t \in F_T} f(\w_t,\x_{Y_t})}{\prob{\tau = T, F_T = \mathbf{y}}} \right]\right]\\
    &= \mathbb{E}\left[\sum_{\tau = T, F_T = \mathbf{y}} \chi(\tau = T, F_T = \mathbf{y}) \mathbb{E}_{\mu\times\cD^T}\left[\frac{\chi(\tau = T, F_T = \mathbf{y})\sum_{t \in \mathbf{y}} f(\w_t,\x_{y_t})}{\prob{\tau = T, F_T = \mathbf{y}}} \right]\right]\\
    &= \mathbb{E}\left[\sum_{\tau = T, F_T = \mathbf{y}} \chi(\tau = T, F_T = \mathbf{y}) \mathbb{E}_{\mu}\left[\frac{\chi(\tau = T, F_T = \mathbf{y})\mathbb{E}_{\cD^T}[\sum_{t \in \mathbf{y}} f(\w_t,\x_{y_t})]}{\prob{\tau = T, F_T = \mathbf{y}}} \right]\right]\\
    &= \mathbb{E}\left[\sum_{\tau = T, F_T = \mathbf{y}} \chi(\tau = T, F_T = \mathbf{y}) \mathbb{E}_{\mu}\left[\frac{\chi(\tau = T, F_T = \mathbf{y})\mathbb{E}_{\cD^T}[\sum_{t \in \mathbf{y}} F(\w_t)]}{\prob{\tau = T, F_T = \mathbf{y}}} \right]\right]\\
    &= \mathbb{E}\left[\sum_{\tau = T, F_T = \mathbf{y}} \chi(\tau = T, F_T = \mathbf{y}) \mathbb{E}\left[\frac{\chi(\tau = T, F_T = \mathbf{y})\sum_{t \in \mathbf{y}} F(\w_t)}{\prob{\tau = T, F_T = \mathbf{y}}} \right]\right]\\
    &=\mathbb{E}\left[\mathbb{E}\left[\sum_{t\in F_\tau} F(\w_t)|\tau, F_\tau\right]\right] \\
    & = \mathbb{E}\left[\sum_{t\in F_\tau} F(\w_t)\right] \\
    & \geq n\mathbb{E}[F(\hat \w_\tau)],
\end{align*}
where in the second equality, we condition on $\tau, F_\tau$, and write the conditional expectation as total expectation as: $\mathbb{E}\left[\sum_{t\in F_\tau} f(\w_t,\x_{\Y_t})|\tau,F_\tau\right] = \mathbb{E}\left[\frac{\chi(\tau = T, F_T = \mathbf{y})\sum_{t \in F_T} f(\w_t,\x_{Y_t})}{\prob{\tau = T, F_T = \mathbf{y}}} \right]$.
In the fourth equality, we use the fact that the iterates $\w_t$ are fixed, and take the expectation of $f(\w_t,\z_{y_t})$ with respect to the data generating distribution $\cD$, yielding $F(\w_t)$. The rest of the steps collapses the conditional expectation back to a total expectation, and finally
the last inequality holds using convexity. Using Corollary~\ref{cor:reg_bound} we get
\begin{align*}
    2D(L + \sigma\sqrt{d})\sqrt{n}
    &\geq \mathbb{E}\left[\sum_{t\in F_\tau} f(\w_t,x_{Y_t})\right] + \mathbb{E}\left[\sum_{t=1}^\tau \langle\xi_t, \w_t \rangle\right] - \mathbb{E}\left[\left(\sum_{t \in F_\tau} f(\u,\x_{y_t}) + \sum_{t=1}^\tau \langle \xi_t,\u\rangle\right)\right]\\
    &\geq n\mathbb{E}[F(\hat \w_{\tau})] + \mathbb{E}[\tau]\times 0 - n F(\u) + \mathbb{E}[\tau]\times 0 = n(\mathbb{E}[F(\hat \w_{\tau})] - F(\u)),
\end{align*}
where in the second inequality we have used Wald's lemma to guarantee $\mathbb{E}[\sum_{t=1}^\tau \langle\xi_t, \w_t \rangle] = 0$ and $\mathbb{E}[\sum_{t=1}^\tau \langle\xi_t, \u \rangle] = 0$. 

\end{proof}

Corollary~\ref{cor:reg_bound} with Lemma~\ref{lem:stopping_time_conv} gives the utility guarantee. 
\begin{theorem}
\label{thm:utility_guarantee}
Suppose the elements in $\cW$ are bounded in norm by $D$ and that $f(\cdot,\x_t)$ is an $L$-Lipschitz function for all $t \in [n]$. Running  Algorithm~\ref{alg:private_SGD} for $\tau$ iterations with noise sequence $\xi_t\sim \mathcal{N}(0, \sigma\I)$ and step size $\eta_t = \frac{D}{\sqrt{n}(L + \sigma\sqrt{d})}$ guarantees that
\begin{align*}
    \mathbb{E}[F(\hat\w_\tau)] - F(\w^*) \leq \frac{5}{2}\frac{D(L + \sigma\sqrt{d})}{\sqrt{n}}.
\end{align*}
\end{theorem}

\section{Privacy Proof}
\label{sec:privacy_proof}
In this section, we establish the differential privacy of Algorithm \ref{alg:private_SGD}. The privacy proof essentially follows the analysis of noisy-SGD in \cite{bassily2014private}, but stated in full detail, for completeness and to provide a self-contained presentation. The basic idea is as follows. We view Algorithm~\ref{alg:private_SGD} as a variant of noisy stochastic gradient descent on the ERM problem over $n$ samples  $\{\x_i\}_{i=1}^n$. Indeed, but for the step where we update the iterate only with noise and do not compute a gradient, the proposed algorithm would be exactly equivalent to noisy SGD. Prior work shows that using noisy SGD to solve the ERM problem enjoys nicer differential privacy due to privacy amplification via subsampling. Intuitively,  Algorithm~\ref{alg:private_SGD} should also benefit from privacy amplification, and the algorithm should not suffer from privacy loss in steps where we use the zero gradient. We formalize this intuition using the standard analysis for R\`enyi differential privacy~\citep{wang2019subsampled} and properties of R\`enyi divergence \citep{mironov2017renyi}.

We first introduce additional notation. Let $\cM_i: \cW \times [n]^* \times S \rightarrow \cW \times [n]^*$ be a function which describes the $i^{\text{it}}$ iteration of the algorithm -- it takes as input, an iterate $\w_{i} \in \cW$, a set, $F_{i} \subseteq [n]^*$, of indices of data points, and a subset $S_i\subseteq S$; it outputs an iterate $\w_{i+1} \in \cW$ and set of indices $F_{i+1} \subseteq [n]^*$. We assume that $S$ is totally ordered according to the indices of the datapoints. Further let $\Sfrak_i$ denote the set of indices of datapoints in $S_i$. Let $\cM_i^1$ and $\cM_i^2$ denote the first and second outputs of $\cM_i$, i.e., $\cM_i^1(\w_{i},F_{i},S) = \w_{i+1}$ and  $\cM_i^2(\w_{i},F_{i},S) = F_{i+1}$. Note that in the algorithm, the $\cM_i$'s are composed -- we initialize $F_1 = \emptyset$ and $\w_1$ is fixed, therefore, $\cM_1(\w_1,\emptyset, S) = (\w_2, F_2)$. $\cM_2$ acts on the output of $\cM_1(\w_1,\emptyset, S)$ as $\cM_2((\cM_1(\w_1,\emptyset, S)),S)=(\w_3,F_3)$. In general, we have that $$\cM_i(\cM_{i-1}(\cdots (\cM_1(\w_1,\emptyset,S_1))\cdots),S) = (\w_{i+1},F_{i+1}).$$ 
Formally, given $\w\in \mathcal{W},m\in\mathbb{N}$, and $F\in [n]^*$, we define random variable $\cM_i(\w,F,S): \bbR^d \times [n]^m \rightarrow \cW \times [n]^*$ as 
\begin{align*}
    \mathcal{M}_{i}(\w_i,F_i,S)(\xi_i,\Sfrak_i) = \big[\mathcal{P}\left(\w_i - \eta_i(\g_i+\xi_i)\right), F_i \bigcup \Sfrak_i\big],
\end{align*}
where $\g_i$ is the following gradient operator:
\begin{align*}
    \g_i = \begin{cases}
        \frac{1}{|\Sfrak_i\setminus F_i|}\sum_{j \in \Sfrak_i\setminus F_i} \nabla f(\w_i,\x_j) & \text{if} \ \ |\Sfrak_i\setminus F_i|>0\\
        0 &\text{if}  \ \ |\Sfrak_i\setminus F_i|=0.
    \end{cases}
\end{align*}
In Algorithm~\ref{alg:private_SGD}, $m=1$, and in general $m$ is just the size of the sub-sampled set from $S$, which we use to construct a mini-batched gradient. 

The input space of $\mathcal{M}_{i}(\w_i,F_i,S)$, which is $ \bbR^d \times [n]^m $ has measure $\cN(0,\sigma^2\bbI) \times \text{Unif}([n],m)$, where $ \text{Unif}([n],m)$ is the sub-sampling measure, which sub-samples $m$ out of $n$ data points uniformly randomly. We now construct a vector concatenating the \emph{first} outputs of all these $\cM_i$'s. Let 
\begin{align*}
    \mathbf{M}_\tau(\w_1, S)&= [\w_1, \w_2, \cdots \w_\tau]\\
    &= [\w_1, \cM_1^1(\w_1,\emptyset), \cM_2^1((\cM_1^1(\w_1,\emptyset),S),S),\\
    &\qquad \qquad \cdots, \cM_\tau^1(\cM_{\tau-1}(\cdots (\cM_1(\w_1,\emptyset,S)))\cdots),S)].
\end{align*} 

We claim that $\mathbf{M}_\tau(\w_1, S)$ is differentially private.

\begin{theorem}
\label{thm:dp_guarantee}
Suppose we run Algorithm~\ref{alg:private_SGD} with noise sampled from $\mathcal{N}(\bf 0, \sigma^2\mathbf{I}_d)$, where $\sigma = \frac{L\sqrt{3\log{1/\delta}}}{\tilde \epsilon}$. For a fixed $\tau$ and $\tilde \epsilon\leq 1.256$, $\mathbf{M}_{\tau}(\w_1,S)$ is $(\frac{2\tilde \epsilon\sqrt{2\tau\log{1/\delta'}}}{|S|} + \frac{4\tau}{|S|^2}\tilde \epsilon^2, \frac{\tau}{|S|}\delta + \delta')$-DP.
\end{theorem}

To prove the above theorem we are first going to show that the first coordinate of each $\mathcal{M}_i$ is sufficiently differentially private.

\begin{lemma}
\label{lem:amp_subsamp}
Let $\sigma = \frac{L\sqrt{3\log{1/\delta}}}{\tilde \epsilon}$.
For any $\w,F$ and any $i$ the mechanism $\mathcal{M}^1_i(\w,F, S)$ is $\left(\frac{m}{|S|}(e^{\tilde \epsilon}-1), \frac{m}{|S|}\delta\right)$-DP.
\end{lemma}
\begin{proof}
Consider the two differing datasets $S$ and $S'$ and let the index at which they differ be $\rho$. Let $p = m/|S|$. Denote the random subsample from $S$ and $S'$ as $\tilde S$ and $\tilde S'$ respectively  Under the uniform random sampling we have $\tilde \Sfrak = \tilde \Sfrak'$. This holds because the sampling of indices only depends on the size $n$ of $S$ and $S'$. The proof follows the standard privacy amplification by sub-sampling argument \footnote{See \url{https://www.ccs.neu.edu/home/jullman/cs7880s17/HW1sol.pdf}}. For fixed $\w$ and $F$, and any measurable (with respect to the Lebesgue measure on the space $\mathbb{R}^d$) $\mathcal{E}\subseteq \mathcal{W}$ define the measures

\begin{align*}
    P(\mathcal{E}) &= \prob{\mathcal{M}_i^1(\w,F, S) \in \mathcal{E} \vert \rho \in \tilde\Sfrak,\rho \not\in F}\\
    P'(\mathcal{E}) &= \prob{\mathcal{M}_i^1(\w,F, S') \in \mathcal{E} \vert \rho \in \tilde\Sfrak,\rho \not\in F}\\
    Q(\mathcal{E}) &= \prob{\mathcal{M}_i^1(\w,F, S) \in \mathcal{E}\vert \rho \in F}\\
    Q'(\mathcal{E}) &= \prob{\mathcal{M}_i^1(\w,F, S') \in \mathcal{E}\vert \rho \in F}\\
    R(\mathcal{E}) &= \prob{\mathcal{M}_i^1(\w,F, S) \in \mathcal{E}\vert \rho \not\in \tilde\Sfrak,\rho \not\in F}\\
    R'(\mathcal{E}) &= \prob{\mathcal{M}_i^1(\w,F, S') \in \mathcal{E}\vert \rho \not\in \tilde\Sfrak,\rho \not\in F}.
\end{align*}
First we note that $Q(\mathcal{E}) = Q'(\mathcal{E})$ because the gradient descent step is restricted only to points indexed by $[\tilde S]\setminus F$ and the differing point $\rho$ does not belong to that set. We also have $R(\mathcal{E}) = R'(\mathcal{E})$ because in this case $\rho$ is not part of the subsampled points.
We consider two cases: $\rho \in F$ and $\rho \not \in F$. In the first case $\prob{\mathcal{M}_i^1(\w,F, S) \in \cE} = Q(\cE)= Q'(\cE)=\prob{\mathcal{M}_i^1(\w,F, S') \in \cE}$, and we have perfect $(0,0)$-DP. In the other case, we have
\begin{align*}
    \prob{\mathcal{M}_i^1(\w,F,S) \in \mathcal{E}} - p\delta &= pP(\mathcal{E}) + (1-p)R(\mathcal{E}) - p\delta\\
    &\leq p(e^{\tilde \epsilon}\min\{P'(\mathcal{E}),R(\mathcal{E})\})) + p\delta + (1-p)R(\mathcal{E}) - p\delta\\
    &= p(\min\{P'(\mathcal{E}),R(\mathcal{E})\} + (e^{\tilde \epsilon}-1)\min\{P'(\mathcal{E}),R(\mathcal{E})\}) + (1-p)R(\mathcal{E})\\
    &\leq p(P'(\mathcal{E}) + (e^{\tilde \epsilon}-1)(pP'(\mathcal{E}) + (1-p)R(\mathcal{E}))) + (1-p)R(\mathcal{E})\\
    &= pP'(\mathcal{\tilde \epsilon}) + (1-p)R(\mathcal{E}) + p(e^{\tilde \epsilon}-1)(pP'(\mathcal{E}) + (1-p)R(\mathcal{E}))\\
    &=  (1+p(e^{\tilde \epsilon} - 1))(pP'(\mathcal{E}) + (1-p)R(\mathcal{E}))\\
    &\leq e^{p(e^{\tilde \epsilon} - 1)}(pP'(\mathcal{E}) + (1-p)R(\mathcal{E}))\\
    &=e^{p(e^{\tilde \epsilon} - 1)}\prob{\mathcal{M}_i^1(\w,F,S')\in \mathcal{E}},
\end{align*}
where in the first inequality we have used the following DP-guarantee of $\mathcal{M}^1(\w,F,S)$. For any subsampled sets $\tilde S \subseteq S$ and $\tilde S'\subseteq S'$ it holds that the mechanism $\mathcal{M}^1$ is $(\tilde \epsilon,\delta)$-DP as it is a step of noisy projected gradient descent with gradients bounded in norm by $L$ and Gaussian noise with sufficient variance. We use the DP guarantee twice -- once on the neighboring datasets $S$ and $S'$ to get the inequality $P(\mathcal{E}) \leq e^{\tilde \epsilon}P'(\mathcal{E})$ and once on the neighboring dataset $S\setminus \{\x_\rho\}$ to get the inequality $P(\mathcal{E}) \leq e^{\tilde \epsilon}R(\mathcal{E})$. In the second application, note that $R(\mathcal{E})$ contains events when a previously seen point is sampled; however the noise-only-step ensures that it also is a Gaussian with the same  variance.
In the second inequality we use the fact that a convex combination of two numbers is greater than their minimum and in the last inequality we use the standard relation $1+x \leq e^x$.
Combining the two cases finishes the proof.
\end{proof}
We can now prove Theorem~\ref{thm:dp_guarantee}.
\begin{proof}[Proof of Theorem~\ref{thm:dp_guarantee}]
The proof essentially follows the proof of Theorem 3.3~\citep{DRV10}. Let $\epsilon_0 = \frac{1}{|S|}(e^{\tilde \epsilon} -1)$, $\delta_0 = \frac{1}{|S|}\delta$, let $\epsilon_1 = e^{\epsilon_0} - 1$ and let $\epsilon' = \epsilon_0\sqrt{2\tau\log{1/\delta'}} + \tau \epsilon_0\epsilon_1$. Condition on the event that throughout the $\tau$ rounds, $\epsilon_0$-DP holds for all of the mechanisms $\mathcal{M}_i^1$. This event fails with probability $\tau \delta_0$ and will be accounted for in the final DP bound. Define the set of events on which $\epsilon'$-DP does not holds as 
\begin{align*}
    \mathcal{B} = \{\mathcal{E} \colon \prob{\mathbf{M}_\tau(\w_1,S)\in \mathcal{E}} > e^{\epsilon'}\prob{\mathbf{M}_\tau(\w_1,S')\in \mathcal{E}}\}.
\end{align*}
The goal is to show that $\prob{\mathbf{M}_\tau(\w_1,S) \in \mathcal{B}} \leq \delta'$. Let $\mathbf{W} = \mathbf{M}_\tau(\w_1,S) = [\W_1,\ldots,\W_\tau]$, where $\W_i$ is the random variable taking values $\w_i$, and let $\mathbf{W}' = \mathbf{M}_\tau(\w_1,S')$. Further let $\w$ be a fixed realization of $\mathbf{W}$, i.e., $\mathbf{w} = [\w_1,\ldots,\w_\tau]$. We have
\begin{align*}
    \log{\frac{\prob{\mathbf{W} = \mathbf{w}}}{\prob{\mathbf{W}' = \mathbf{w}}}} &= \sum_{i=1}^\tau \log{\frac{\prob{\W_i = \w_i \vert \W_{i-1} = \w_{i-1},Y_{i-1}=y_{i-1},\ldots, \W_{2} = \w_2,Y_2=y_2, Y_1=y_1 }}{\prob{\W'_i = \w_i \vert \W'_{i-1} = \w_{i-1},Y_{i-1} = y_{i-1}, \ldots, \W'_{2} = \w_2, Y_2 = y_2, Y_1= y_1 }}}\\
    &+ \log{\frac{\prob{Y_1=y_1}}{\prob{Y_1=y_1}}}.
\end{align*}
Consider the conditional probability $\prob{\W_i = \w_i \vert Y_i = y_i,\W_{i-1} = \w_{i-1},Y_{i-1}=y_{i-1},\ldots, \W_{2} = \w_2,Y_2=y_2}$. After conditioning on all the randomness so far in the algorithm the event $\{\W_i = \w_i\}$ is just the event that $\{\mathcal{M}^1_{i-1}(\w_{i-1},F_{i-1},S) = \w_i\}$, where $\w_{i-1}$ and $F_{i-1}$ fixed. Lemma~\ref{lem:amp_subsamp} now implies that
\begin{align*}
    \frac{\prob{\W_i = \w_i \vert \W_{i-1} = \w_{i-1},Y_{i-1}=y_{i-1},\ldots, \W_{2} = \w_2,Y_2=y_2, Y_1=y_1 }}{\prob{\W'_i = \w_i \vert \W'_{i-1} = \w_{i-1},Y_{i-1} = y_{i-1}, \ldots, \W'_{2} = \w_2, Y_2 = y_2, Y_1= y_1 }} \leq e^{\epsilon_0},
\end{align*}
recall we conditioned on the event that each of the mechanisms $\mathcal{M}_i^1$ are $\epsilon_0$-DP.
Define the random variable $c_{i-1}(\W_2,\ldots,\W_i,\W_2',\ldots,\W_i',Y_1,\ldots,Y_{i-1})$ which takes values 
\begin{align*}
    c_{i-1}(\w_2,\ldots,\w_i,y_1,\ldots,y_{i-1}) = \log{\frac{\prob{\W_i = \w_i \vert \W_{i-1} = \w_{i-1},Y_{i-1}=y_{i-1},\ldots, \W_{2} = \w_2,Y_2=y_2, Y_1=y_1 }}{\prob{\W'_i = \w_i \vert \W'_{i-1} = \w_{i-1},Y_{i-1} = y_{i-1}, \ldots, \W'_{2} = \w_2, Y_2 = y_2, Y_1= y_1 }}}.
\end{align*}
We have just shown that $c_{i-1}(\w_2,\ldots,\w_i,\w_2,\ldots,\w_i,y_1,\ldots,y_{i-1}) \leq \epsilon_0$ for any $\w_{2:i},y_{1:i-1}$. By switching $\W_i$ with $\W_i'$ we can show in the exact same way that $-c_{i-1}(\w_{2:i},\w_{2:i},y_{1:i-1}) \leq \epsilon_0$ which implies that the random variable $c_{i-1}(\W_{2:i},\W_{2:i}', Y_{1:i-1})$ is a.s. bounded by $\epsilon_0$. Further using the fact that $\epsilon_0$-DP implies boundedness in the $\infty$-divergence we also have 
\begin{align*}
\max\big\{D_{\infty}(\mathcal{M}_i^1(\w_i,F_i,S)||\mathcal{M}_i^1(\w_i,F_i,S')), D_{\infty}(\mathcal{M}_i^1(\w_i,F_i,S')||\mathcal{M}_i^1(\w_i,F_i,S))\big\} \leq \epsilon_0.
\end{align*}
Lemma 3.2 in \citep{DRV10} now implies that the KL-divergence between $\mathcal{M}_i^1(\w_i,F_i,S)$ and $\mathcal{M}_i^1(\w_i,F_i,S')$ is bounded as $D_{KL}(\mathcal{M}_i^1(\w_i,F_i,S)||\mathcal{M}_i^1(\w_i,F_i,S')) \leq \epsilon_1\epsilon_0$.
Together with the definition of $c_{i-1}$ this implies 
\begin{align*}
\mathbb{E}_{\mathbf{W}_i,\mathbf{W}_i'}[c_{i-1}(\W_{2:i},\W_{2:i}', Y_{1:i-1})| \w_{2:i-1},y_{1:i-1}] = D_{KL}(\mathcal{M}_i^1(\w_i,F_i,S)||\mathcal{M}_i^1(\w_i,F_i,S')) \leq \epsilon_0\epsilon_1.
\end{align*}
Because the filtration generated by $\W_{2:i-1},Y_{1:i-1}$ includes the one generated by $c_{1:i-1}$ we can now apply Azuma's inequality on the martingale difference sequence $c_{i-1}(\W_{2:i},\W_{2:i}',Y_{1:i-1}) - D_{KL}(\mathcal{M}_i^1(\w_i,F_i,S)||\mathcal{M}_i^1(\w_i,F_i,S'))$, which we have just shown it is bounded as the KL-term is bounded by $\epsilon_0\epsilon_1$ and $c_i$'s are bounded by $\epsilon_0$, and get
\begin{align*}
    \prob{\mathbf{M}_\tau(\w_1,S) \in \mathcal{B}} &= \prob{\sum_{i=1}^\tau c_{i-1}(\W_{2:i},\W_{2:i}',Y_{i-1}) > \epsilon'}\\
    &= \prob{\sum_{i=1}^\tau c_{i-1}(\W_{2:i},\W_{2:i}',Y_{i-1}) - \tau\epsilon_0\epsilon_1 > \epsilon_0\sqrt{2\tau\log{1/\delta'}} }\\
    &\leq \exp{(-2\epsilon_0^2 \tau\log{1/\delta'})/(2\tau\epsilon_0^2\epsilon_1^2)} = \delta'.
\end{align*}
Because for $\tilde \epsilon \leq 1.256$ it holds that $e^{\tilde \epsilon} -1 \leq 2\tilde \epsilon$ we have $\epsilon_0 \leq \frac{2\tilde \epsilon}{|S|}$. This in hand implies $\epsilon_1 \leq 2\epsilon_0$ which shows that for a fixed $\tau$ Algorithm~\ref{alg:private_SGD} is $(\frac{2\tilde \epsilon\sqrt{2\tau\log{1/\delta'}}}{|S|} + \frac{4\tau}{|S|^2}\epsilon^2, \frac{\tau}{|S|}\delta + \delta')$-DP.
\end{proof}

Combining the utility and privacy guarantees yields the main result, stated below.

\begin{theorem}
\label{thm:main_result}
Let $f(\w,\x)$ be a convex $L$-Lipschitz function in its first argument, $\w \in \cW$, for all $\x \in \cZ$. Furthermore, assume that the diameter of $\cW$ is bounded by $D$. For any $n\geq16$, given $0<\epsilon\leq \frac{1}{2\sqrt{n}}$, $\delta >0$, Algorithm \ref{alg:private_SGD} run with step size $\eta = \frac{D}{\sqrt{n}(L+\sigma\sqrt{d})}$ and $\sigma = \frac{8L\sqrt{\log{1/\delta}}}{\sqrt{n}\epsilon}$ outputs $\hat \w_\tau$ which satisfies $(4\epsilon(\sqrt{\log{1/\delta'}}+2), \delta + \delta' + 2\exp{-n/16})$ -- differential privacy, for any $\delta'>0$, and its accuracy is bounded as, 
\begin{align*}
    \E{F(\hat \w_\tau) - F(\w^*)} \leq \frac{5LD}{\sqrt{n}} + \frac{20LD\sqrt{d\log{1/\delta}}}{\epsilon n}.
\end{align*}
\end{theorem}
\begin{proof}
Condition on the event that $\tau \leq 2n$. Using Theorem~\ref{thm:dp_guarantee}, with $\tilde \epsilon = \sqrt{n}\epsilon$ implies that Algorithm~\ref{alg:private_SGD} is $(4\epsilon(\sqrt{\log{1/\delta'}}+2), \delta + \delta')$-DP under the condition that $\epsilon \leq \frac{1.256}{\sqrt{n}}$. The event $\tau > 2n$ happens with probability $-2\exp{-n/16}$ according to Theorem~\ref{thm:bins_and_balls}, which implies that Algorithm~\ref{alg:private_SGD} is $(4\epsilon(\sqrt{\log{1/\delta'}}+2), \delta + \delta' + 2\exp{-n/16})$-DP. Finally, the convergence guarantee follows from Theorem~\ref{thm:utility_guarantee}.
\end{proof}

\begin{remark}
To get $(\bar \epsilon,\bar\delta)$-DP, for any $6\exp{-n/16} \leq \bar\delta \leq 3e^{-4}$, we can set $\delta' = \delta = \frac{\bar \delta}{3}$, and $\epsilon = \frac{\bar \epsilon}{8 \sqrt{\log{1/\delta'}}}$. Using the condition $\epsilon \leq \frac{1}{\sqrt{n}}$, we get that $\frac{\bar \epsilon}{\sqrt{\log{3/\bar \delta}}} \leq \frac{8}{\sqrt{n}}$ The utility guarantee becomes
$  \E{F(\hat \w_\tau) - F(\w^*)} \leq \frac{5LD}{\sqrt{n}} + \frac{160LD\sqrt{d}\log{3/\bar \delta}}{\bar \epsilon n}$.
\end{remark}

\section{Other Approaches}
\label{sec:other_appr}
We now briefly discuss a few other approaches which would also recover a convergence rate of $\tilde O(1/\sqrt{n} + \sqrt{d}/(\epsilon n))$ for $\epsilon \leq O(\sqrt{1/n})$. 
The first approach uses the standard decomposition of excess population risk to generalization gap $+$ excess empirical risk.
As pointed out in \cite{bassily2019private,feldman2019private}, it is possible to use the optimal rates for ERM from \cite{bassily2014private} together with generalization propties of DP from \cite{DFHPR15a} to bound the generalization gap, to guarantee a rate of $O\left(\max\left(d^{1/4}/\sqrt{n}, \sqrt{d}/(n\epsilon)\right)\right)$, which evaluates to $\sqrt{d/n}$ in the regime $\epsilon \leq O(1/{\sqrt{n}})$. Even though the runtime of noisy-SGD for ERM, stated in \cite{bassily2014private} is $O(n^2)$, it can be shown that their result holds with only $O\br{\frac{(n\epsilon)^2}{\sqrt{d}}}$ stochastic gradient computations -- therefore, in the regime $\epsilon \lesssim \frac{1}{\sqrt{n}}$, it is a linear time algorithm.

Second, it may be possible to use amplification by subsampling in Algorithm 1 of \cite{feldman2019private} instead of amplification by iteration. This would discard the smoothness requirement for $f(\cdot,\x)$ and perhaps yield the same result in Theorem~\ref{thm:main_result}. We note that Algorithm~\ref{alg:private_SGD} is much simpler and does not require the complicated mini-batch schedule from Algorithm 1 of \cite{feldman2019private}.

\section{Conclusion}
In this paper, we studied the problem of private stochastic convex optimization for non-smooth objectives. We proposed a noisy version of the OSMD algorithm and presented convergence and privacy guarantees for the $\ell_2$ geometry. Algorithm~\ref{alg:private_SGD} achieves statistically optimal rates, while running in linear time, and is differentially private as long as the DP parameter is sufficiently small. Algorithm~\ref{alg:private_SGD} can easily be extended to geometries induced by arbitrary strictly convex potentials $\Phi$. We leave it as future work to explore what kind of privacy guarantees can be retained if the privacy mechanism is tailored towards the geometry induced by $\Phi$. Finally, we think it should be possible to extend our analysis to the case when $f(\cdot,\x_i)$ is strongly convex for all $i\in[n]$ and achieve optimal statistical rates in linear time.

\bibliographystyle{plainnat}
\bibliography{mybib}

\newpage
\appendix

\section{Proof of Theorem~\ref{thm:bins_and_balls}}
\begin{proof}
We begin by fixing a time horizon $T$ and showing that with high probability $\tau \leq T$ for appropriately chosen $T$. Let $q:[n]^T \rightarrow \mathbb{N}$ be the function which counts the unique draws among $(Y_1,\ldots,Y_T)$. Further, let $Z_i$ be the indicator random variable of the event that the $i$-th datapoint is drawn at least ones in the $T$ rounds, i.e., $\exists j \in [T] : Y_j = i$. We can write the expectation of $q$ as 
\begin{align*}
    \mathbb{E}[q(Y_1,\ldots,Y_T)] &= \mathbb{E}\sum_{i=1}^n Z_i = \sum_{i=1}^n \mathbb{E}[Z_j] = \sum_{i=1}^n \prob{\exists j \in [T] : Y_j = i}\\
    &= \sum_{i=1}^n 1-\prob{\forall j \in [T], Y_j \neq i} = \sum_{i=1}^n 1 - \left(1 - 1/n\right)^T \geq n - n\exp{-T/n}.
\end{align*}
Setting $T \geq n$ already shows that the number of unique elements after $T \geq n$ is at least $n/2$ and thus the stopping time condition is reached. Next, we show concentration around the expectation of $q$. Note that if a single element $y_i$ is changed by $y_i'$ the value of $q$ will not change by more than $1$ i.e.,
\begin{align*}
    |q(y_1,\ldots,y_T) - q(y_1,\ldots,y_{i-1},y_i',y_{i+1},\ldots y_T)| \leq 1.
\end{align*}
This allows us to use McDiarmid's inequality to claim that
\begin{align*}
    \prob{q(Y_1,\ldots,Y_T) - \mathbb{E}[q(Y_1,\ldots,Y_T)] < -\sqrt{\log{2/\delta}T/2}} \leq \delta. 
\end{align*}
This implies that with probability at least $1-\delta$ we have
\begin{align*}
    q(Y_1,\ldots,Y_T) \geq n-n\exp{-T/n} - \sqrt{\log{2/\delta}T/2}.
\end{align*}
Setting $T=2n$ and $\delta = 2\exp{-n/16}$ we have that for any $n\geq 16$ it holds that with probability $1- 2\exp{-n/16}$ 
\begin{align*}
    q(Y_1,\ldots,Y_T) \geq n - ne^{-2} - \sqrt{n\times n/16} \geq n/2.
\end{align*}
To get the bound in expectation we note that $\prob{\tau \geq T} = \prob{q(Y_1,\ldots,Y_T) \leq n/2}$. On the other hand we know from the above derivation, that for $T \geq 2n$ it holds that $\prob{q(Y_1,\ldots,Y_T) \leq n/2} \leq 2 e^{-T/16}$. This implies
\begin{align*}
\mathbb{E}[\tau] \leq 2n\prob{\tau \leq 2n} + \int_{2n}^\infty \prob{\tau \geq t} dt \leq 2n + \int_{2n}^\infty 2\exp{-t/16} dt \leq O(n + ne^{-n/16}).
\end{align*}
\end{proof}

\end{document}